\documentclass{article} 
\usepackage[final]{nips_2017}
\pdfoutput=1 
\usepackage[utf8]{inputenc} 
\usepackage[T1]{fontenc}    
\usepackage{hyperref}       
\usepackage{url}            
\usepackage{booktabs}       
\usepackage{amsfonts}       
\usepackage{nicefrac}       
\usepackage{microtype}      
\usepackage{todonotes}
\usepackage{amsmath}
\usepackage{subfig}
\usepackage{float}
\usepackage{amsthm}
\usepackage{graphics}

\newtheorem{theorem}{Theorem}
\newtheorem{lemma}[theorem]{Lemma}
\newtheorem{definition}{Definition}[section]
\newtheorem{prop}{Proposition}

\let\OLDthebibliography\thebibliography
\renewcommand\thebibliography[1]{
  \OLDthebibliography{#1}
  \setlength{\itemsep}{3pt}
}

\title{Collaborative Filtering with Side Information: a Gaussian Process Perspective}

\author{
  Hyunjik Kim \\
  Department of Statistics\\
  University of Oxford\\
  \texttt{hkim@stats.ox.ac.uk} \\
   \And
  Xiaoyu Lu \\
  Department of Statistics\\
  University of Oxford\\
  \texttt{xiaoyu.lu@stats.ox.ac.uk} \\
   \And
  Seth Flaxman \\
  Department of Statistics\\
  University of Oxford\\
  \texttt{flaxman@stats.ox.ac.uk} \\
   \And
  Yee Whye Teh \\
  Department of Statistics\\
  University of Oxford\\
  \texttt{y.w.teh@stats.ox.ac.uk} \\
}

\begin{document}

\maketitle

\begin{abstract}
We tackle the problem of collaborative filtering (CF) with side information, through the lens of Gaussian Process (GP) regression. Driven by the idea of using the kernel to explicitly model user-item similarities, we formulate the GP in a way that allows the incorporation of low-rank matrix factorisation, arriving at our model, the \textit{Tucker Gaussian Process} (TGP). Consequently, TGP generalises classical Bayesian matrix factorisation models, and goes beyond them to give a natural and elegant method for incorporating side information, giving enhanced predictive performance for CF problems. Moreover we show that it is a novel model for regression, especially well-suited to grid-structured data and problems where the dependence on covariates is close to being separable.

\end{abstract}

\section{Introduction}

Collaborative filtering (CF) defines a branch of techniques for tackling the following supervised learning problem: making predictions (filtering) about the preferences of a user, based on information regarding the preferences of many users (collaboration). We are given data in the form of a partially observed rating matrix $R$, where $R_{ij}$ is the rating of user $u_i$ on movie $v_j$ for $i=1,\ldots,n_1, j=1,\ldots,n_2$. CF aims to predict missing entries of $R$ by only using the observed entries. Hitherto, matrix factorisation approaches \citep{billsus1998learning,koren2009bellkor,piotte2009pragmatic,toscher2009bigchaos} have been the basis for many successful CF models. These model $R$ as a product of two low rank matrices $R \approx UV^\top$, hence $R_{ij} \approx \sum_{k} U_{ik} V_{jk}$. On the other hand, content-based filtering predicts user ratings based on attributes of users (e.g. age, sex) and items (e.g. genre). CF with side information is a combination of the two, aiming to predict user ratings using both ratings data and user/item attributes. 

There has been a wide range of work on CF with side information, mostly building on the framework of matrix factorisation. Suppose user/item side information is given in the form of feature matrices $F=[\omega(u_1),...,\omega(u_{n_1})]^\top \in \mathbb{R}^{n_1 \times r}$ and $G=[\omega'(v_1),...,\omega'(v_{n_2})]^\top \in \mathbb{R}^{n_2 \times r}$. \textit{Matrix co-factorization} \cite{singh2008relational} attempt to factorise $F,G$ and $R$ simultaneously, whereas the \textit{Regression-based Latent Factor Model} \citep{agarwal2009regression} assumes instead that $U$ and $V$ are linear in $F$ and $G$. \textit{Bayesian Matrix Factorization with Side Information} (BMFSI) \citep{porteous2010bayesian} gives an additive model in the sense that $R$ is assumed to be the sum of the standard matrix factorisation prediction $UV^\top$ and linear contributions of $F$ and $G$. \textit{Hierarchical Bayesian Matrix Factorization with Side Information} \citep{park2013hierarchical} is an extension of BFMSI with Gaussian-Wishart hyperpriors on the prior mean and variance of $U$ and $V$.

Gaussian Processes (GPs) are a popular class of Bayesian nonparametric priors over functions \citep{rasmussen2006gaussian}, and have served as flexible models across a range of machine learning tasks, e.g. regression, classification, dimensionality reduction \citep{lawrence2004gaussian} and CF \citep{lawrence2009non,yu2006stochastic}. In a regression setting with input and output pairs, a key advantage of GPs is that we can use the kernel to explicitly model similarity in the outputs between a pair of input values. We use this to model similarity between users/items given side information, forming the outset of the paper. We model the ratings as $R_{ij} \sim \mathcal{N}(f(u_i,v_j),\sigma^2)$ and $f \sim \mathcal{GP}(k_1 \times k_2)$ where kernels $k_1$ and $k_2$ model user and item similarities respectively. However, a direct GP regression application is infeasible due to its $O(N^3)$ computational cost, because for most CF problems, the number of ratings $N$ ranges from a hundred thousand to hundreds of millions. Low-rank matrix factorisation remedies this problem, and also underlies the state of the art approaches for CF. Hence it is natural to look for a connection between GPs and low-rank matrix factorisation, which is the motivation and contribution of our work.

To develop a framework for this connection, we first propose a novel approximation scheme for GPs. Our starting point is the Kronecker structure that arises naturally when working with kernels that are products of simpler constituent kernels (say each dependent on one covariate dimension). Coupled with the weight space view of GPs, we can represent a draw from the GP as a product between a random weight tensor and a collection of feature vectors (one for each constituent kernel). The weight tensor can be very large for high dimensional problems, and our proposal is to approximate it using a low-rank Tucker decomposition \citep{tucker1966some} instead. This reduces the effective number of parameters that need to be learnt, and forms the link between GPs and matrix factorisation methods in CF. Thus we arrive at our model, the Tucker Gaussian Process (TGP).

We make the following contributions:
\begin{itemize}
    \item TGP is an elegant and effective method for modelling user/item similarities via kernels to exploit side-information. As far as we know, ours is the first work to use GPs for modelling similarities via side information, with explicit correspondence between similarities and the kernel.
    \item TGP generalises classical Bayesian Matrix factorisation models \citep{SalMnihICML08,SalMnih08}, bridging the gap between matrix factorisation methods and GP methods in CF.
    \item Sub-linear scaling of TGP, achieved by stochastic gradient descent, makes it suitable for CF problems that typically have large data sets infeasible for GPs.
    \item TGP is applicable to certain regression problems where the regression function is separable in the covariates. We reason that the Tucker decomposition acts as a regulariser for the GP that helps control overfitting, and verify that TGP outperforms GPs in generalisation performance for these problems.
\end{itemize}


\textbf{Outline of paper} We formulate TGP in the general regression setting in Section \ref{sec:tgp}, and describe its central application to CF with side information in Section \ref{sec:CF}, followed by related work and discussion in Section \ref{sec:disc}. We present experimental results in Section \ref{sec:experiments} and conclude in Section \ref{sec:conclusion}.

\section{Tucker Gaussian Process Regression}
\label{sec:tgp}
\subsection{Tucker GP Regression}
Consider a regression problem with inputs $x_1,\ldots, x_N \in \mathcal{X}$ and corresponding observations $y_1, \ldots y_N \in \mathbb{R}$. We assume $y_i|x_i \sim \mathcal{N}(f(x_i),\sigma^2)$ for some $f:\mathcal{X} \rightarrow \mathbb{R}$ and that the observations are independent. The aim is to learn $f$. One approach is to put a Gaussian Process (GP) prior on $f$, with zero mean and covariance $k$. The training then consists of computing the posterior GP. The problem of using this in a CF setting is that training costs $O(N^3)$ operations. $N$, the number of ratings, usually ranges from a hundred thousand to hundreds of millions in CF, making inference infeasible.

The weight space view of GPs offers a natural way of dealing with the problem: suppose there exists a feature map $\phi:\mathcal{X} \rightarrow \mathbb{R}^n$ (where $n$ is the number of features) such that $k(x,x')=\phi(x)^\top\phi(x') \hspace{1mm} \forall x,x' \in \mathcal{X}$. Then the GP is equivalent to Bayesian Linear Regression with feature vectors used for each row of the design matrix \citep{rasmussen2006gaussian}:
\begin{align} \label{eq:weightspace}
\begin{split}
    y|x,\theta \overset{ind}{\sim} \mathcal{N}(f(x),\sigma^2), &\hspace{3 mm} f(x)=\theta^\top \phi(x) \\
    \theta \sim \mathcal{N}(0,I), &\hspace{3 mm} \theta \in \mathbb{R}^n
\end{split}
\end{align}
Now training takes $O(Nn^2)$ time, and is scalable for $n \ll N$.

\begin{figure}
    \centering
    \subfloat[Tensor]{\label{fig:tensor}\includegraphics[width=40mm]{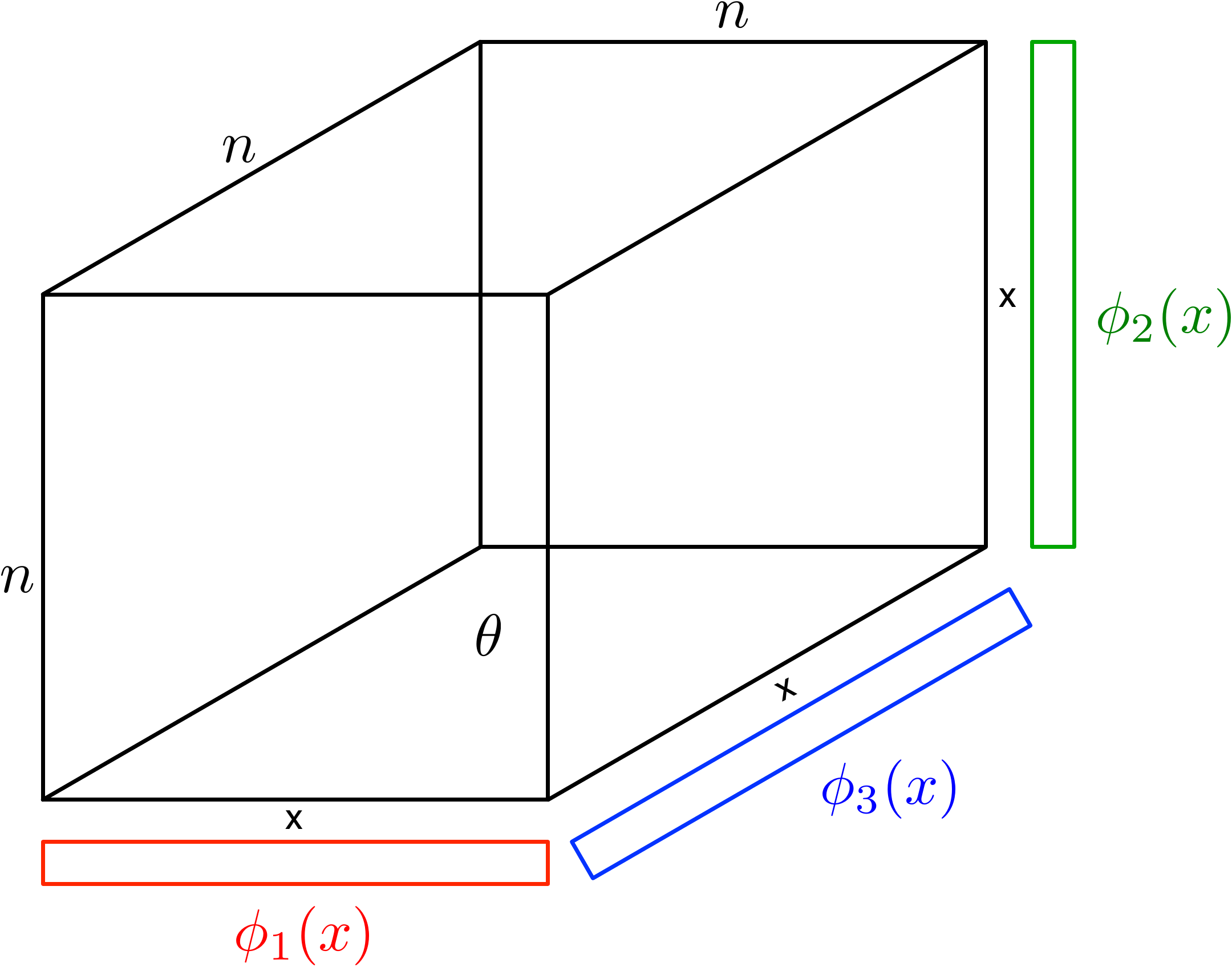}}
    \subfloat[Tucker]{\label{fig:tucker}\includegraphics[width=40mm]{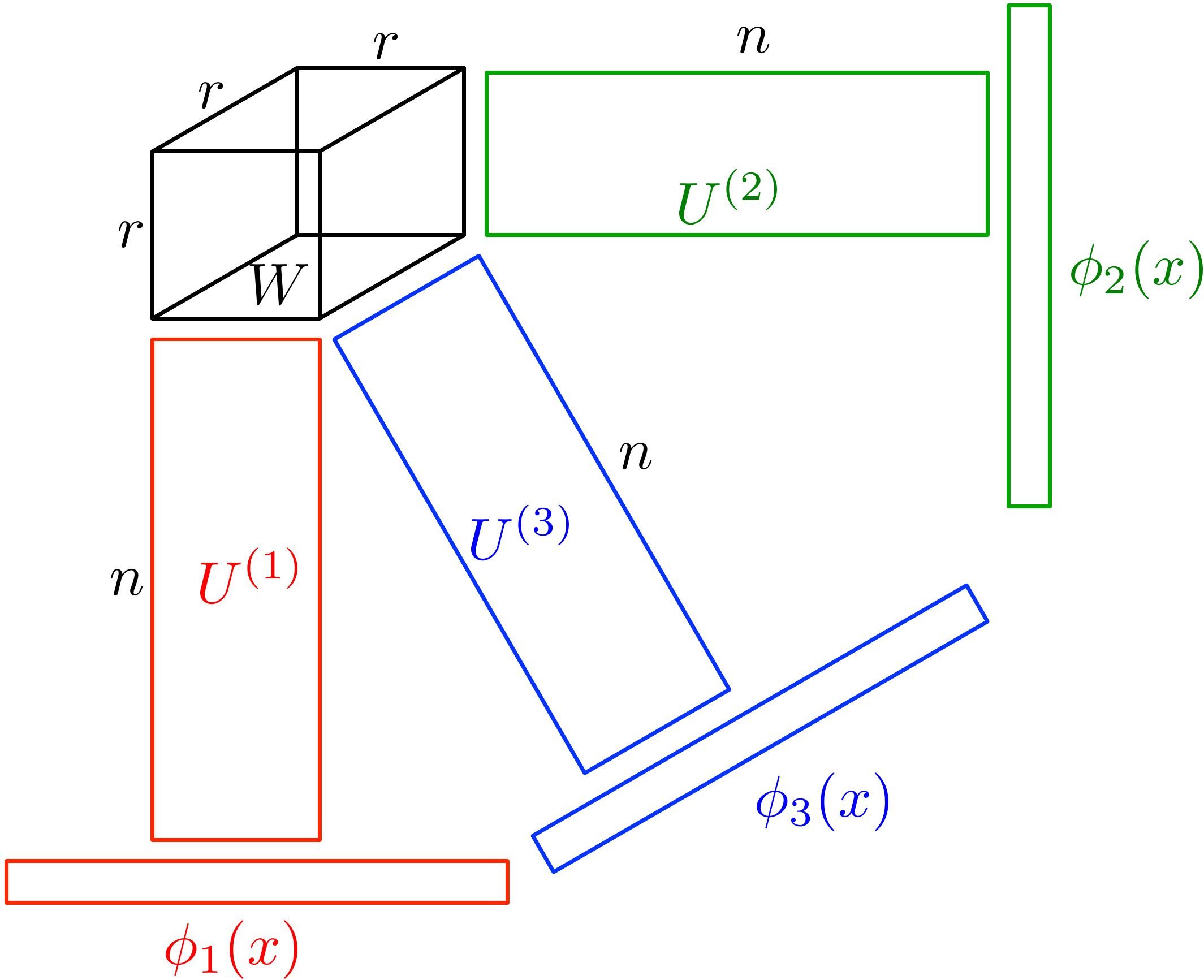}}
    \caption{Tensor \& Tucker Decomposition representation of regression function for $D=3$.}
\vspace*{-1em}
\end{figure}

Consider the case of product kernels, where the kernel can be written as follows:
\begin{equation} \label{eq:prodkernel}
    k(x_i,x_j)=\prod_{d=1}^D k_d(x_i,x_j)
\end{equation}
and suppose there are feature maps $\phi_d:\mathcal{X} \rightarrow \mathbb{R}^n$ such that $k_d(x_i,x_j) = \phi_d(x_i)^\top\phi_d(x_j)$. Then we can write $k(x_i,x_j) =  \phi(x_i)^\top \phi(x_j)$ where $\phi(x)=\otimes_{d=1}^D \phi_d(x)$ is the Kronecker product of the $\phi_d$. Returning to (\ref{eq:weightspace}):
\begin{equation} \label{eq:fullrank}
f(x) =\theta^\top \phi(x) = \theta^\top \big(\otimes_{d=1}^D \phi_d(x)\big) = \theta \times_{d=1}^D \phi_d(x)
\end{equation}
where $\theta$ has been reshaped as a $D$-dimensional tensor in $\mathbb{R}^{n \times \ldots \times n}$ in the rightmost expression, as in Figure \ref{fig:tensor}. We define the tensor product notation as follows:
\begin{align*}
\begin{split}
\theta \times_{d=1}^D \phi_d:=& \hspace{2mm} vec(\theta)^\top \otimes_{d=1}^D \phi_d 
 = \sum_{i_1,\ldots,i_D=1}^n \theta_{i_1,\ldots,i_D} \prod_{d=1}^D(\phi_d)_{i_d}
\end{split}
\end{align*}
We refer to (\ref{eq:fullrank}) as the \textit{full-rank} model, and use $\theta$ as a tensor for the rest of the paper.

This \textit{full-rank} model is problematic in high dimensions: the size of $\theta$ grows as $n^D$, so the function computation become infeasible. Thus we introduce the novel \textit{Tucker Gaussian Process} (TGP) model, where we circumvent this problem by approximating $\theta$ using a low-rank Tucker decomposition \citep{tucker1966some}. This is a tensor-matrix product between a low rank core tensor $W \in \mathbb{R}^{r \times \ldots \times r}$ of dimension $D$ and matrices $U^{(1)}, \ldots U^{(D)} \in \mathbb{R}^{n \times r}$, as in Figure \ref{fig:tucker}. We denote $\theta \approx W \times_{d=1}^D U^{(d) \top}$ where the $(i_1,\ldots,i_D)^{th}$ entry is $W \times_{d=1}^D U_{i_d}^{(d)}$ with $U_{i_d}^{(d)}$ a column vector representing the $i_d^{th}$ row of $U^{(d)}$. $n$ is the number of features in each dimension and $r$ is the rank. Note that we are free to use a different $n$ and $r$ for each dimension, but assume these are the same for convenience of notation.

We must also place suitable priors on $W$ and $U^{(d)}$ to match the iid $\mathcal{N}(0,1)$ prior on each entry of $\theta$. We place iid priors $\mathcal{N}(0,1)$ on each entry of $W$, and $\mathcal{N}(0,\sigma_u^2)$ on each entry of $U^{(d)}$. Setting $\sigma_u^2=\frac{1}{r}$, we match the first two moments of $W \times_{d=1}^D U^{(d) \top}$ and $\theta$. We then prove in Appendix \ref{apd:conv} that each entry of $W \times_{d=1}^D U^{(d) \top}$ converges in distribution to $\mathcal{N}(0,1)$ as $r \rightarrow \infty$.

In summary our TGP regression model approximating data from a GP with  product kernel (\ref{eq:prodkernel}) and homoscedastic noise is:
\begin{gather} \label{eq:tgp}
     y|x \overset{ind}{\sim} \mathcal{N}(f(x),\sigma^2),\hspace{1mm} f(x)=W\times_{d=1}^D \big(U^{(d) \top}\phi_d(x)\big)
\end{gather}
where $\phi_d:\mathcal{X} \rightarrow \mathbb{R}^n$ are feature maps such that $k_d(x_i,x_j)=\phi_d(x_i)^\top\phi_d(x_j)$, and we have iid $\mathcal{N}(0,1)$ and $\mathcal{N}(0,\frac{1}{r})$ priors on the entries of $W$ and $U^{(d)}$ respectively.

\subsection{Choice of Feature Map}
\label{sec:feature}
So far, we have assumed that the kernels $k_d$ can be written as the inner product of feature vectors: $k_d(x_i,x_j)=\phi_d(x_i)^\top \phi_d(x_j)$. We investigate the situations where this assumption holds. When this doesn't hold, we explore other choices of $\phi$ that approximate $k_d$.

\textbf{Identity features} One case where we can write kernels as inner products of features is with identity kernels $k_d(x_i,x_j)=\delta_{ij}$. The features are unit vectors: $\phi_d(x_i) = e_i := (0, \cdots, 0, 1, 0, \cdots )^\top$ with the non-zero at the $i^{th}$ entry, hence $U^{(d)\top}\phi_d(x_i) =U_i^{(d)}$. However this implies $U^{(d)} \in \mathbb{R}^{N \times r}$ (or $\mathbb{R}^{n_d \times r}$ for inputs on a grid), so for $N$ (or $n_d$) too big, computations can become too costly both in time and memory. A workaround is to use feature hashing \citep{weinberger2009feature} to obtain shorter features whose inner products are unbiased estimates of inner products of the original features. This technique can be applied to arbitrary features where the number of features is too large. See Appendix \ref{apd:hash} for details.

We can also deal with cases where the data lies on a grid using Cholesky features, or in the most general case where the data doesn't lie on a grid and $k_d$ cannot be expressed as the inner product of finite feature vectors using Random Fourier features. See Appendix \ref{apd:feature} for details.

\subsection{Learning}
\label{sec:learning}
In TGP we would like to learn the posterior distribution of $U$ and $W$. The simplest and fastest method of learning is Maximum a Posteriori (MAP), whereby we approximate the posterior with point estimates $\hat{U},\hat{W}=\arg\max_{U,W} p(U,W|y)$. For the optimisation we may use stochastic gradient descent (SGD) to approximate the full gradient, for which we get a time complexity of $O(m(nrD+r^D D))$ operations for computing the stochastic gradient on a mini-batch of size $m$, which is sublinear in $N$. See Appendix \ref{apd:learning} for a details. 

The problem with a MAP estimate for $U,W$ is that only the posterior mode is used, and the uncertainty encoded in the shape of the posterior distribution is ignored. In a Bayesian setting, we wish to use samples from the posterior and average predictions over samples. For data where we can afford an $O(N)$ runtime, we may use sampling algorithms such as Hamiltonian Monte Carlo (HMC) \citep{duane1987hybrid,neal2011mcmc}. The runtime for each HMC leapfrog step is $O(N(nrD+r^D D))$, the same time complexity as a step of full-batch gradient descent.

\section{TGP for Collaborative Filtering with Side Information} \label{sec:CF}

In order to apply TGP to CF, let us first formulate the problem using GPs. It is natural to model this as a supervised regression problem with $ R_{ij} \sim \mathcal{N}(f(u_i,v_j),\sigma^2)$ and prior $f \sim \mathcal{GP}(0,k)$ \citep{yu2006stochastic}. Note that this is particularly suitable with side information, since kernels can be interpreted as measures of similarity; we can design $k$ to encode similarities between users/movies given by the side information. Hence we may further exploit the use of GPs for addressing this problem. In particular we use a product kernel $k((u_i,v_j),(u_{i'},v_{j'})) = k_1(u_i,u_{i'})k_2(v_j,v_{j'})$ since we expect similar ratings for two user/movie pairs if the users are similar \textit{and} the movies are similar. When there is no side information, it is sensible to use identity kernels $k_1(u_i,u_{i'})=\delta_{u_iu_{i'}}$, $k_2(v_j,v_{j'})=\delta_{v_jv_{j'}}$. i.e. that distinct users and movies are not similar a priori. With side information, we may add on further kernels $\kappa_1,\kappa_2$ modelling similarity between users/movies: $k_1(u_i,u_{i'})=a_1^2\delta_{u_i u_{i'}}+b_1^2 \kappa_1(u_i,u_i')$, $k_2(v_j,v_{j'})=a_2^2\delta_{v_j v_{j'}}+b_2^2 \kappa_2(v_j,v_j')$, where $a$ and $b$ are parameters controlling the extent to which similarity in side information leads to similarity in preference. 

However, it is not immediately clear how this single GP framework relates to the matrix factorisation approach. We show that our proposed TGP forms a natural connection between these two approaches, and that we recover classic matrix factorisation models as a special case. To apply TGP, first note that we have $D=2$, and the Tucker Decomposition is simply a low-rank matrix factorisation. Using the notation $U,V$ instead of $U^{(1)},U^{(2)}$, we have that $\theta \approx UWV^\top$, hence $f(u_i,v_j)=\phi_1(u_i)^\top U W (\phi_2(v_j)^\top V)^\top$. With the identity kernel, we have unit vector features $\phi_1(u_i) = e_i \in \mathbb{R}^{n_1}$ and $\phi_2(v_j)=e_j \in \mathbb{R}^{n_2}$. TGP therefore simplifies to:
\begin{gather} \label{eq:tgpcf}
    R_{ij} \overset{ind}\sim \mathcal{N}(f(u_i,v_j),\sigma^2),  \hspace{3mm} f(u_i,v_j) = U_i^\top W V_j
\end{gather}
with iid $\mathcal{N}(0,\sigma_u^2)$ priors on each entry of $U,V$ where $U_i, V_j$ are column vectors representing the $i^{th}$ and $j^{th}$ row of $U$ and $V$ respectively. Note that with $W=I$ fixed, we recover Probabilistic Matrix Factorization (PMF) \citep{SalMnih08}, a particularly effective Bayesian model in the matrix factorization framework. 

An extension is Bayesian PMF (BPMF) \citep{SalMnihICML08} where priors are placed on the prior mean and covariance of $U_i,V_j$. Should we decide to learn $W$ in TGP, interesting parallels arise between our model and BPMF. Observe from the following that learning $W$ can be a proxy for learning the prior mean and covariance of $U$ and $V$, as is done in the BPMF model:

\begin{figure}[!htbp]
\centering
  \includegraphics[width=0.7\linewidth]{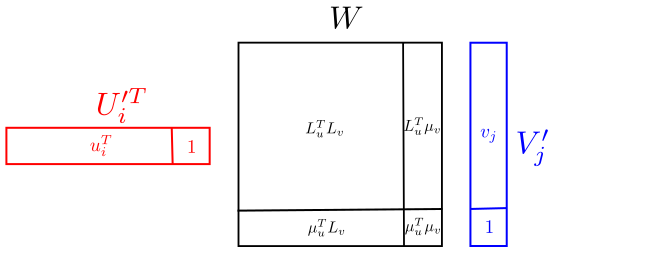}
  \caption{Bayesian PMF reparametrised. \label{fig:reparam}}
  \vspace*{-3em}
\end{figure}

\begin{gather*}
U_i \sim \mathcal{N}(\mu_u,\Lambda_u), V_j \sim \mathcal{N}(\mu_v,\Lambda_v) \Rightarrow U_i=\mu_u+L_u u_i, V_j=\mu_v+L_v v_j \\
\text{where } u_i,v_j \sim \mathcal{N}(0,I), \Lambda_u=L_u L_u^\top, \Lambda_v=L_v L_v^\top \\
\Rightarrow U_i^\top V_j =\mu_u^\top\mu_v+\mu_u^\top L_v v_j + u_i^\top L_u^\top \mu_v + u_i^\top L_u^\top L_v v_j = U_i'^\top W V'_j
\end{gather*}
where $U_i'^\top = [u_i^\top,1]$, $W=[L_u^\top L_v,L_u^\top \mu_v;\mu_u^\top L_v,\mu_u^\top \mu_v]$, $V_j'=[v_j;1]$, as displayed in Figure \ref{fig:reparam}. So a full $W$ with standard iid Gaussian priors on $U,V$ can capture the effects of modelling $U,V$ with non-zero means and full covariances for each row of $U,V$, as in BPMF.

Returning to the case with side information, suppose it is given in the form of vectors $\omega_1(u_i), \omega_2(v_j)$, and that we expect users/movies with similar $\omega$ to show similar preferences/be preferred by similar users. For example we can encode the user age into $\omega_1$ and the movie genre into $\omega_2$ and define $\kappa_1(u_i,u_i')=\omega_1(u_i)^\top \omega_1(u_{i'}),\kappa_2(v_j,v_j')=\omega_2(v_j)^\top\omega_2(v_{j'})$. The feature vector is now $\phi_d(u_i) =  [a_d e_i^\top, b_d\omega_d(u_i)^\top]^\top $ for $d=1,2$, and we have
$f(u_i,v_j) = \phi_1(u_i)^\top U W V^\top \phi_2(v_j)$.

\section{Related Work and Discussion}
\label{sec:disc}

Modelling data in the form of matrices and tensors has been studied in the field of multi-way data analysis and relational learning. The key idea here is to factorise the data tensor, with two notable forms of factorisation: PARAFAC \citep{bro1997parafac} and Tucker \citep{tucker1966some}. There are a few works in these domains that relate to GPs. InfTucker \citep{xu2011infinite} uses the Tucker decomposition directly on the data tensor, and use a non-linear transformation of the parameters $U^{(d)}$ for the regression function, contrary to TGP which is linear in the parameters. DinTucker \citep{zhe2016dintucker} tries to scale up InfTucker by splitting up the observed tensor into subarrays. \citep{imaizumi2015doubly} motivate their model using the Parafac decomposition instead of Tucker, expressing the regression function as a sum of products of local functions. These local functions are each modelled by GPs. However the TGP is motivated from a single GP on the input space. \citep{suzuki2016minimax} again model the regression function as a sum of product of local functions, which live in the RKHS of some kernel, analogous to the feature maps in TGP. However there is no mention of GPs or how their model relates to low-rank tensor decomposition. \citep{zhao2013tensor} deals with the classification problem where each input is a tensor, so there is one label per tensor. They define a GP over the space of tensors. It is unclear whether they actually use low rank tensor decomposition. For TGP we deal with regression, work in the setting where each element of the data tensor corresponds to a response, and apply Tucker decomposition to the parameters. \citep{zhe2016distributed} define a GP over the parameter space whereas TGP is a multilinear expression in the parameters and feature maps, approximating a GP over the input space. In short, these models make completely different assumptions to TGP, and thus are useful for different CF applications - none use side information (it is unclear how this would be possible given their model assumptions) and do not relate to the Bayesian matrix factorisation literature.

There are closer connections between our model and the Stochastic Relational Model \citep{yu2006stochastic} in relational learning. It is a special case of our model with $W=I$ and $D=2$. The key differences lie in the inference: we use features to build on the weight-space view of GPs, whereas \cite{yu2006stochastic} work with GPs in the function-space view. This complicates learning for kernels which cannot be expressed as an inner product of features; the authors resort to Laplace approximation for finding maximum likelihood estimates of parameters. For such kernels we use random feature maps (see Appendix \ref{apd:feature}), making learning simple and more computationally efficient.

In the domain of matrix factorisation, \cite{lawrence2009non} use a GP Latent Variable Model (GP-LVM) \citep{lawrence2004gaussian}. They learn a latent vector for each movie, and pass it through a zero-mean GP with squared exponential (SE) kernel, with one GP per user. In TGP we use one GP for all users and items. For a CF application, they incorporate side information about movies by taking the product of these kernels with a SE kernel in the movie features. Our model is more flexible in that we can take into account both user and item similarities simultaneously.

From the perspective of GP regression, we analyse the regression function of TGP to understand the regression problems for which it will be effective. Recall that the regression function $f(x)$ in (\ref{eq:tgp}) can be seen as $W\times_{d=1}^D \psi_d(x)$, where $\psi_d(x)=U^{(d) \top}\phi_d(x)$ are lower-dimensional features in $\mathbb{R}^r$ (i.e. the $U^{(d)}$ multiplied by $\phi_d(x)$ in Figure \ref{fig:tucker}). With this new formulation, we have:
\begin{equation} \label{eq:ssf}
f(x)= W \times_{d=1}^D \phi_d(x) = \sum_{i_1,\ldots,i_D = 1}^r W_{i_1 \ldots i_D}\prod_{d=1}^D (\psi_d(x))_{i_d}
\end{equation}
Hence learning $W$ and $\big(U^{(d)}\big)_{d=1}^D$ can be interpreted as learning features $\psi_d$ as well as their weights for the regression function, i.e. learning a linear combination of products of these features. In the case where $\psi_d(x)$ is only a function of the $d^{th}$ dimension of $x$, each $\prod_{d=1}^D (\psi_d(x))_{i_d}$ is separable in the dimensions. Modelling data with sums of separable functions has been studied in \cite{beylkin2009multivariate}, and its effectiveness for regression is shown by promising results on various synthetic and real data. Such additive models arise frequently in the context of ensemble learning, such as \textit{boosting} and BART \citep{chipman2010bart}, where a linear combination of many weak learners is used to build a single strong learner. We may interpret our model in this framework where $\prod_{d=1}^D(\psi_d(x))_{i_d}$ are the weak learners that share parameters, and $W_{i_1 \ldots i_D}$ are the corresponding weights.

With this alternative interpretation in mind, we may expect TGP to perform well in cases where the data displays an additive structure, with the additive components arising from a product of features on each dimension. Hence we interpret TGP as a modified GP where the approximation acts as a regulariser towards such simpler functions, which can actually lead to enhanced generalisation performance by controlling overfitting. We thus compare its performance to GPs on spatio-temporal data sets where it is reasonable to expect separability in longitude and latitude, or in time and space. 

Based on Section \ref{sec:feature}, we also see that our model is particularly well-suited to modelling grid-structured data. The difference between our model and that in \cite{saatcci2012scalable} is that we have Kronecker structure in the features $\phi$, whereas they exploit Kronecker structure on the data. Moreover, our model can deal with data not on a grid, as well as data on a grid with many missing observations, since observations are not needed for constructing the features.

Going back to CF, recall from Section \ref{sec:CF} that the low-rank matrix factorisation model has $R_{ij} \approx \sum_{k} U_{ik} V_{jk}$, a sum of a product of parameters(features) in each dimension. TGP generalises this to modelling a linear combination of products of features, hence we may expect it to perform well for this task. Also note the grid structure, since users and items are categorical variables.
\section{Experimental Results}
\label{sec:experiments}

\textbf{Collaborative Filtering} We use the MovieLens 100K data\footnote{Obtained from \url{http://grouplens.org/datasets/movielens/100k/}}, which consists of 100,000 ratings in $\{1,\ldots,5\}$ from 943 users on 1682 movies. User age, gender and occupation are given, as well as the genre of the movies. We represent this side information with binary vectors for $\omega_1(u_i),\omega_2(v_j)$ and use the formulation in (\ref{eq:cfside}) in Appendix \ref{apd:cf}. We bin the age into five categories, and there are 20 occupations and 18 genres. Thus $\omega_1(u_i) \in \mathbb{R}^{5+2+20}$ has 3 non-zero entries, one for each feature, and $\omega_2(v_j) \in \mathbb{R}^{18}$ can have multiple non-zero entries since each movie can belong to many genres. We report the mean and standard deviation of the test RMSE on the five 80:20 train test splits that come with the data, as it will offer a sensible means of comparison with other algorithms. $N$ is too large for HMC, hence we use SGD to obtain MAP estimates for the parameters, and compare different configurations: learning $W$/fixing it to be the identity and using/not using side information, along with BPMF initialised by PMF\footnote{Code obtained from \url{http://www.cs.toronto.edu/~rsalakhu/BPMF.html}}. We use mini-batches of size 100, and set $r=15$ for all models as it gives best results for PMF and BPMF. We used a grid search and cross-validation for tuning hyperparameters, the recommended method in big $N$ settings where the number of hyperparameters is not too large. SGD was not so sensitive to mini-batch size, and finding the range of hyperparameters was straightforward. See Appendix \ref{apd:cf} for details.

\begin{table}[h!]
\vspace*{-1.5em}
\begin{minipage}{\linewidth}
\centering
\caption{Test RMSE on MovieLens100K.}
\begin{tabular}{|l|l|l|}
\hline
{\bf Model}                     &{\bf Test RMSE} \\ \hline
BPMF                            &$0.9024 \pm 0.0050$      \\ 
TGP, $W=I$ (PMF)                &$0.9395 \pm 0.0115$      \\  
TGP, learn $W$                  &$0.9270 \pm 0.0097$      \\ 
TGP, $W=I$, side-info           &$0.9014 \pm 0.0061$      \\ 
TGP, learn $W$, side-info       &\textbf{0.8995 $\pm$ 0.0062}      \\ \hline
\end{tabular}
\label{tab:ml100k}
\end{minipage}
\end{table}

From Table \ref{tab:ml100k} it is evident that TGP makes good use of side information, since the RMSE decreases significantly with side information. Learning $W$ instead of fixing it helps predictive performance, but does not perform as well as BPMF. One reason is that our Gaussian prior on $W$ is not equivalent to the Gaussian-Wishart priors on the mean and variance of $U_i,V_j$ in BPMF. Another reason is that we are resorting to a MAP estimate. If we can instead sample from the posterior and average predictions over these samples, we expect enhanced predictions. However, note that using TGP with side information and learning $W$, we are able to get comparable/superior results to BPMF, even with a MAP estimate. We expect further improvements not only with sampling but also by using more sophisticated kernels that make better use of the side information; for example, use different hyperparameter coefficients for the different types of features. In so far as comparison was possible, these numbers are comparable to state-of-the-art algorithms in Section \ref{sec:disc}. A direct comparison was not possible as each use different methods for evaluation.

\textbf{Regression on spatial data} We use the California house prices data from the 1990 census\footnote{Obtained from \url{https://www.csie.ntu.edu.tw/~cjlin/libsvmtools/datasets/regression/cadata}}, which consists of average house prices for 20,640 different locations in California. We only use the covariates longitude and latitude, and whiten them along with log-transformed house prices to each have zero mean and unit variance. We chose this data set as spatial data sometimes exhibit separability in the different dimensions. Moreover the data is clustered in urban areas, hence an additive model with each component describing different sections of California may be desirable. Using a random 50:50 train test split, we report the RMSE of the model on the training set and test set after training. We first fit a GP to the data with a squared expeonential (SE) kernel on each dimension using the GPML toolbox \citep{Rasmussen:2010:GPM:1756006.1953029}, optimising the hyperparameters by type-II maximum likelihood. Then using these hyperparameters we generate RFF for $\phi$. See Appendix \ref{apd:rff} and \ref{apd:feature} for details. We implemented both the \textit{full-rank} model and TGP with $n=25,50,100,200$ on Stan \citep{stan-software:2012}, which uses HMC with the No-U-Turn Sampler (NUTS) \citep{hoffman2014no} for inference. Note that for both models $n$ refers to the length of features $\phi_d(x)$.

\begin{figure}[h!]
\vspace*{-1em}
    \centering
    \subfloat[Train RMSE]{\includegraphics[width=60mm]{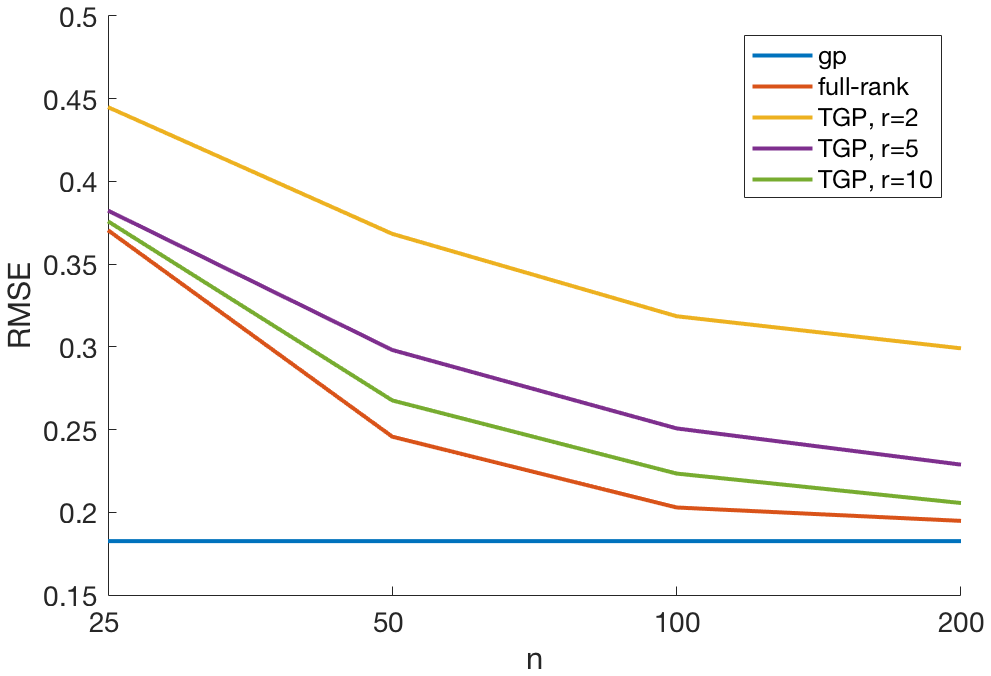}}
    \subfloat[Test RMSE]{\includegraphics[width=60mm]{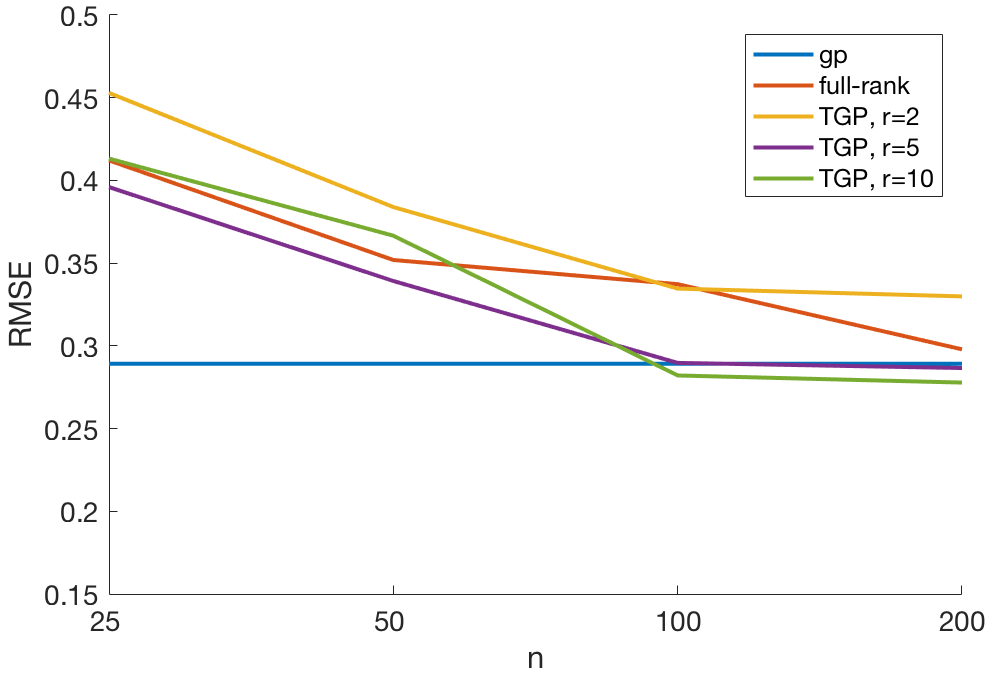}}
    \caption{RMSE for GP, \textit{full-rank}, and TGP for $r=2,5,10$ for $n=25,50,100,200$ on the California House Price data.} \label{fig:cali}
\vspace*{-1em}
\end{figure}

For TGP, we use 300 warmup iterations and a further 300 samples on 4 different chains, and use the mean prediction across the samples. For \textit{full-rank}, we take the same number of samples and chains, but only use 50 warmup draws as we diagnosed that convergence was reached by this point (looking at the Gelman-Rubin statistic \citep{gelman1992inference} and effective sample size). The convergence statistics for TGP are in Appendix \ref{apd:cali}. We can see from Figure \ref{fig:cali} that some TGP models give lower test RMSE and higher train RMSE than the GP and the \textit{full-rank} model. In fact TGP with $r=5$ consistently shows higher predictive performance than \textit{full-rank} for all values of $n$, and for $n \geq 100$ TGP with $r=10$ outperforms GP. This indicates that TGP is an effective regulariser towards simpler regression functions, namely a linear combination of separable functions. We expect bigger gains for TGP with more warmup iterations, since the convergence diagnostics suggest that TGP hasn't quite fully mixed by 300 iterations.

\begin{figure*}[!htbp]
\vspace*{-1em}
    \centering
    \subfloat[True values/TGP Predictions]{\includegraphics[width=46mm]{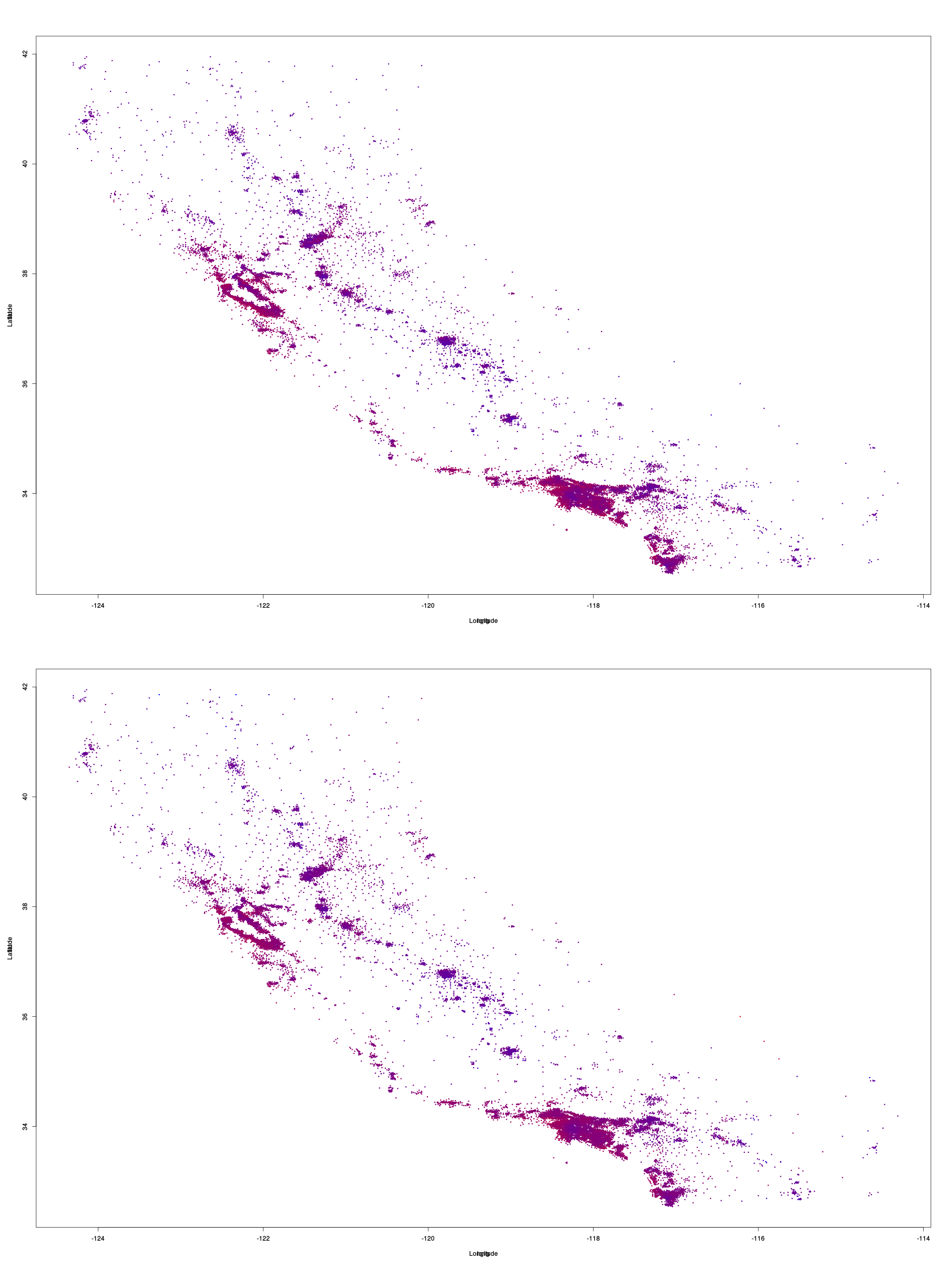}}
    \subfloat[Additive components for TGP]{\includegraphics[width=90mm]{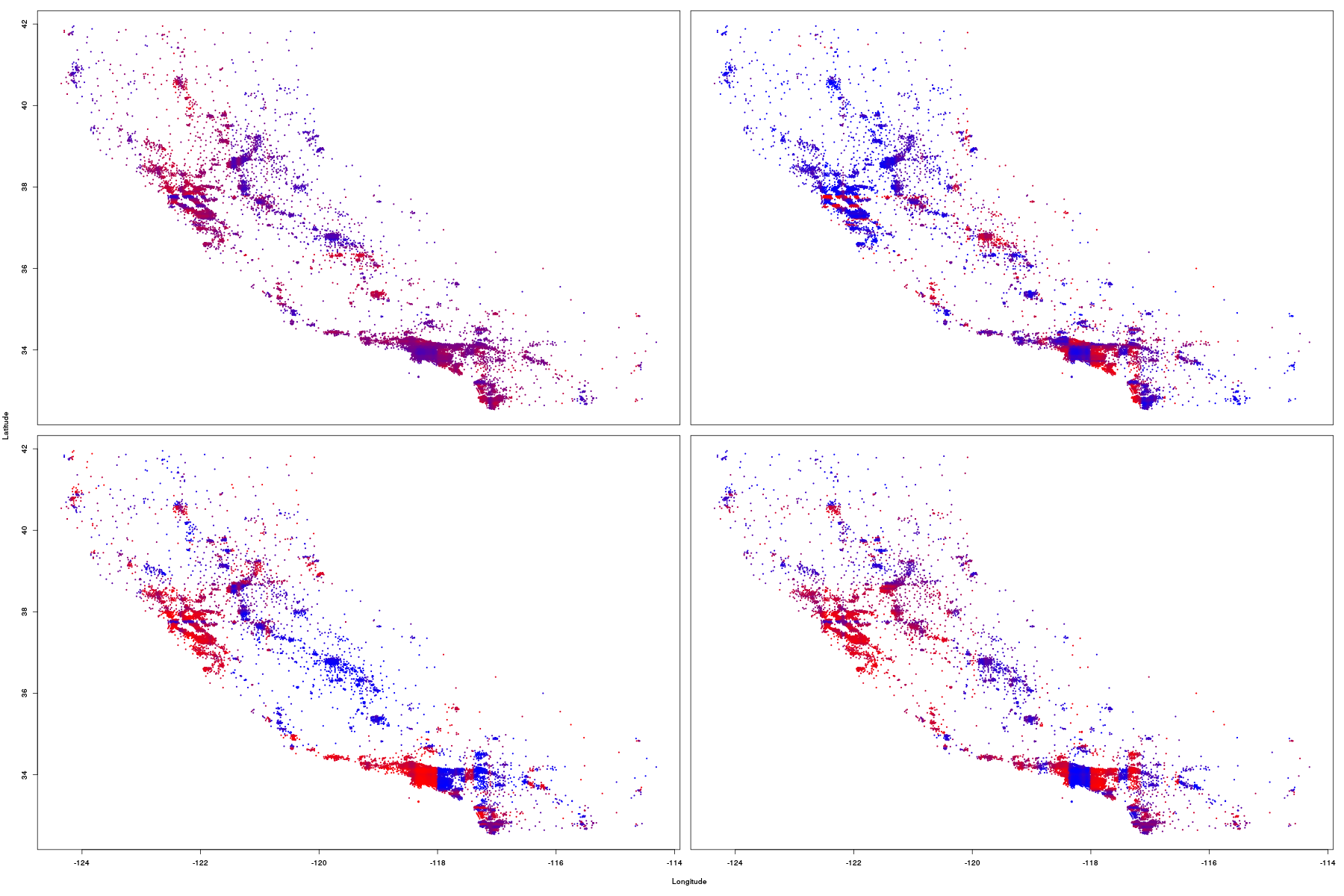} \label{fig:calipred}}
    \caption{(a) Top: Heatmap of true log house price values. Bottom: TGP predictions for $r=2, n=200$. (b) Heatmap showing the four additive components (summands in (\ref{eq:ssf})) of predictions for TGP with $r=2, n=200$. We only use the last sample in the Markov Chain to get a better indication of the structure. Red indicates high log price and blue indicates low, and the same colour scheme is applied to all four subplots. To accentuate the differences in the predictive values, we colour values by the percentile they belong to instead of a uniform colouring. See Appendix \ref{apd:cali} for the uniform colouring.} 
\end{figure*}

We further investigate the predictions of TGP by analysing the additive components in the prediction for $r=2$. We see in Figure \ref{fig:calipred} that the components are quite different. The upper two components show complementary predictions in the Bay area (North-West) and the central area, whereas the bottom two show complementary predictions in the Los Angeles area (South-East). This confirms the hypothesis that the different additive components will learn different sections of the data. See Appendix \ref{apd:cali} for zoomed in plots, and Appendix \ref{apd:wind} for experimental results on spatio-temporal data with grid structure.

\section{Conclusion}\label{sec:conclusion}
We have introduced the Tucker Gaussian Process (TGP), a regression model that regularises a GP towards simpler regression functions, in particular a linear combination of separable functions. We motivate it as a solution to Collaborative Filtering (CF), by using feature maps and a low-rank Tucker decomposition on the parameters in the weight-space view of GPs. In particular, we have highlighted the effectiveness of TGP in CF with side-information, a domain where outputs can be effectively modelled as a linear combination of functions separable in the covariates. We believe that this is the largest contribution of our paper: after showing that PMF and BPMF are special cases of the TGP, we extended it to exploit the user/item side information for better predictions; the kernel of the GP we approximate can be designed to encode similarities between different users and items, a particularly neat and natural method for modelling similarity. In doing so, we bring together matrix factorisation methods and GP methods in CF, as well as scaling up GP methods in CF. We confirm experimentally that side information enhances the predictive performance of TGP in collaborative filtering.

We have also shown that for problems where one might expect separability in the covariates such as prediction for spatio-temporal data sets, the TGP effectively controls overfitting and outperforms GPs in prediction. We also point out that exact Cholesky features can be used with TGP in the case of grid-structured data, and random feature maps can be used for arbitrary kernels.
\section*{Acknowledgments}
HK, XL, SF and YWT's research leading to these results has received funding from the European Research Council under the European Union's Seventh Framework Programme (FP7/2007-2013) ERC grant agreement no. 617071.

\bibliographystyle{plain}
{\small
\bibliography{main.bbl}}
\newpage
\appendix
\newcommand{\forceindent}{\leavevmode{\parindent=2em\indent}}
\section*{Appendix}
\section{Learning Algorithms for TGP} \label{apd:learning}
We give detailed derivations of various inference algorithms for the TGP. We have a set of $N$ observations $y_i \in \mathbb{R}$ corresponding to a set of inputs $x_i \in \mathcal{X}, $ and we wish to regress $y=(y_i)_{i=1}^{N}$ on $X=(x_i)_{i=1}^{N}$. We assume that the data generating mechanism takes the form
\begin{equation*}
y=f(X) + \epsilon \hspace{10 mm} \epsilon \sim \mathcal{N}(0,\sigma^2 I_N)
\end{equation*}
where $f(X)=(f(x_i))_{i=1}^{N} \in \mathbb{R}^N$ and also that the regression function takes the following form

\begin{equation*}
f(x)=w^\top \otimes_{d=1}^D \big(U^{(d)\top}\phi_d(x)\big)
\end{equation*}

where 
\begin{itemize}
\item $W \in \mathbb{R}^{r \times \ldots \times r}$ is a D-dimensional tensor whose entries are iid $\mathcal{N}(0,\sigma_w^2)$
\item $w=vec(W)$ is the vector obtained when flattening tensor $W$, such that $W\times_{d=1}^D v_d=w^\top \otimes_{d=1}^D v_d$ $\forall v_d \in \mathbb{R}^r$
\item $(\phi_d(x))_{d=1}^D $ are the features in $\mathbb{R}^n$ extracted from $x$
\item $(U^{(d)})_{d=1}^D$ are a set of real $n \times r$ matrices with $U^{(d)}_{jl} \overset{iid}{\sim} \mathcal{N}(0,\sigma_u^2)$
\end{itemize}
We assume $n>r$, and wish to learn $w$ and the $U^{(d)}$ from the data. 

Note from the second point that $\nabla_w \big(W\times_{d=1}^D v_d\big)=\otimes_{d=1}^D v_d$. For D=2 for example, if $g(U)=s^\top U t$ for some matrix $U$ and vectors $s,t$, then $\nabla_u g(U) = s \otimes t$ where $u=vec(U)$.

First we give the complexity for calculating $f(x)$. Computing $\psi_d(x_i)=U^{(d) \top}\phi_d(x_i)$ $\forall d$ requires $O(nrD)$ time, then $w^\top \otimes_{d=1}^D \psi_d(x_i)$ takes $O(r^D)$ time. So time for a prediction given $\phi,U,w$ takes $O(nrD+r^D)$.

The quantity of interest for MAP and HMC is the log joint distribution $p(y,U,w)=p(y|U,w)p(U)p(w)$. In full this is:
\begin{equation*}
    \log p(y|U,w)+\log p(U)+\log p(w) = -\frac{1}{2\sigma^2}\sum_{i=1}^N (y_i-f(x_i))^2 - \frac{1}{2\sigma_u^2}\sum_{k=1}^D tr(U^{(k)T}U^{(k)}) - \frac{1}{2\sigma_w^2}w^\top w
\end{equation*}
This has the following derivatives:
\begin{align*}
    \nabla_w\log p(w) &=-w \\
    \nabla_{U^{(k)}}\log p(U) &=-rU^{(k)} \\
    \nabla_w \log p(y_i|U,w) &=\frac{1}{\sigma^2}(y_i-f(x_i)) \otimes_{d=1}^D \psi_d(x_i) \\
    \nabla_{u^{(k)}}\log p(y_i|U,w) &=\frac{1}{\sigma^2}(y_i-f(x_i)) \phi_k(x_i) \otimes \big(W \times_{d \neq k} \psi_d(x_i) \big) \\
\end{align*}
with the following definitions:
\begin{itemize}
\item $u^{(k)}=vec(U^{(k)}) \in \mathbb{R}^{nr}$
\item $(W \times_{d \neq k} v_d)_l := W \times_{d=1}^D v'_d $ where $v'_d=v_d$ for $d \neq l$ and $v'_l=e_l \in \mathbb{R}^r$, the unit vector with non-zero at the $l^{th}$ entry.
\end{itemize}
The last derivative holds since $f(x)=\phi_k(x)^\top U^{(k)}(W \times_{d \neq k} \psi_d(x_i))$. 
Computing $W \times_{d \neq k} \psi_d(x_i)$ takes $O(r^D)$ for each $k$, hence $O(r^D D)$ $\forall k$. So we have that using mini-batches $\{x_{t1},\ldots,x_{tm}\}$ for SGD, we have the following updates for MAP:
\begin{align*}
w &\leftarrow w+\frac{\epsilon^w}{2}\bigg(\nabla_w\log p(w_t)+\frac{N}{m} \sum_{i=1}^m \nabla_w\log p(y_{ti}|x_{ti},w,U) \bigg) \\
u^{(k)} &\leftarrow u^{(k)}+\frac{\epsilon^k}{2}\bigg(\nabla_{u^{(k)}}\log p(U)+\frac{N}{m} \sum_{i=1}^m \nabla_{u^{(k)}}\log p(y_{ti}|x_{ti},w,U) \bigg)
\end{align*}
with time complexity $O(m(nrD+r^D D))$.

Gathering the parameters into a vector $\theta=(w,U^{(1)},\ldots,U^{(k)})$, the HMC algorithm runs as follows:
\begin{enumerate}
\item Initialise MC by drawing $\theta_0=(w_0,U_0^{(1)},...,U_0^{(D)})$ from its prior distribution. 
\item For $t=0,...,T$:
\begin{enumerate}
\item Initialise $p \sim \mathcal{N}(0,I_Q)$, $V^{(k)}\sim \mathcal{N}(0,I_{n \times r}) \hspace{5mm} \forall k$, \\
\mbox{$H_t \leftarrow -\log p(w_t)-\sum_{i=1}^N \log p(y_{i}|x_{i},\theta_t)+\frac{1}{2}\sum_{k=1}^D tr(V^{(k)T} V^{(k)}) +\frac{1}{2}p^T p$} \\
$\theta=(w,U^{(1)},...U^{(D)}) \leftarrow \theta_t=(w_t,U_t^{(1)},...U_t^{(D)})$
\item For $l=1,...,L$:
\begin{enumerate}
\item \mbox{$p\leftarrow p + \frac{\epsilon_t^w}{2} \Big(\nabla_w\log p(w_t)+\sum_{i=1}^N \nabla_w\log p(y_{i}|x_{i},\theta_t)\Big)$} \\
For $k=1,...,D$: \\
\forceindent \mbox{$V^{(k)} \leftarrow V^{(k)}+\frac{\epsilon_t^k}{2}\Big(\sum_{i=1}^N \nabla_{U^{(k)}}\log p(y_{i}|x_{i},\theta) \Big)$} \\
\item $w \leftarrow w + \epsilon_t^w p$ \\
For $k=1,...,D$: \\
\forceindent $u^{(k)} \leftarrow u^{(k)}+\epsilon_t^w V^{(k)}$ \\
\item same as i.
\end{enumerate}
\item \mbox{$H^* \leftarrow -\log p(w)-\sum_{i=1}^N \log p(y_{i}|x_{i},\theta)+\frac{1}{2}\sum_{k=1}^D tr(V^{(k)T} V^{(k)}) +\frac{1}{2}p^T p$} \\
$u \sim Unif[0,1]$ \\
If $u \leq \exp(H_t-H^*)$ \\
\forceindent $\theta_{t+1}=(w_{t+1},U_{t+1}^{(1)},...,U_{t+1}^{(D)}) \leftarrow \theta=(w,U^{(1)},...,U^{(D)})$ \\
else \\
\forceindent $\theta_{t+1} \leftarrow \theta_t$
\end{enumerate}

\end{enumerate}
From previous computations, it is easy to see that each update requires $O(LN(nrD+r^D D))$ operations.

\section{Elementwise convergence of TGP to $\mathcal{N}(0,1)$}
\label{apd:conv}
\theoremstyle{definition}
\begin{definition}{\textbf{Martingale Difference Sequence}}
A martingale difference sequence with respect to a filtration $(\mathcal{F}_p)_{p \in \{0,1, \ldots, r\}}$ is a real-valued sequence of random variables $X_1,\ldots,X_r$ that satisfies: 
\begin{enumerate}
\item $X_p$ is $\mathcal{F}_p$ measurable
\item $\mathbb{E}(|X_p|)< \infty$
\item $\mathbb{E}(X_p|\mathcal{F}_{p-1})=0$ a.s.
\end{enumerate}
for all $p \in \{1, \ldots, r\}$.
\end{definition}

\begin{theorem}[\textbf{Martingale Central Limit Theorem} \citep{hall2014martingale}] 
\label{thm:mclt}
Let $X=\{X_1,\ldots,X_r\}$ be a sequence of random variables satisfying the following conditions:
\begin{enumerate}
\item $X$ is a martingale difference sequence with respect to filtration $(\mathcal{F}_p)_{p \in \{0,1, \ldots, r\}}$
\item $\sum_{p=1}^r \mathbb{E}(X_p^2|\mathcal{F}_{p-1}) \overset{p}{\rightarrow} 1$ as $r \rightarrow \infty$.
\item $\sum_{p=1}^r \mathbb{E}(X_p^2\mathbb{I}(|X_p|>\epsilon)|\mathcal{F}_{p-1}) \overset{p}{\rightarrow} 0$ as $r \rightarrow \infty$ $\forall \epsilon>0$ .
\end{enumerate}
Then the sums $S_r=\sum_{p=1}^r X_p \overset{d}{\rightarrow} \mathcal{N}(0,1)$ as $r \rightarrow \infty$.
\end{theorem}

\begin{prop}
Let $n$ by $r$ matrices $U^{(d)} \overset{iid}{\sim} \mathcal{N}(0,\frac{1}{r}I)$ for $d=1,\ldots,D$, and let $W \sim \mathcal{N}(0,I)$ where $W \in \mathbb{R}^{r \times \ldots \times r}$ is a D-dimensional tensor. Then each element of $W \times_{d=1}^D U^{(d) \top}$ converges in distribution to $\mathcal{N}(0,1)$ as $r \rightarrow \infty$.
\end{prop}
\begin{proof}
Suppose first that $D=2$. It suffices to show that
\begin{gather*}
u,v \in \mathbb{R}^r, W \in \mathbb{R}^{r \times r}, u,v \overset{iid}{\sim} \mathcal{N}(0,I), W \sim \mathcal{N}(0,I) \\
\Rightarrow u^\top W v \overset{d}{\rightarrow} \mathcal{N}(0,1) \text{ as } r \rightarrow \infty    
\end{gather*}
We define for each $r \in \mathbb{N}$:
\begin{align*}
S_0 &= 0 \\
S_p &:= \sum_{i,j=1}^p u_i W_{ij} v_j  \\
X_p &:= S_p - S_{p-1} =  u_p W_{pp}v_p + \sum_{i=1}^{p-1} u_p W_{pi} v_i + u_i W_{ip} v_p  \label{eq:X_p}\\
\mathcal{F}_0 &:= \{\emptyset,\Omega\} \text{ where $\Omega$ is the sample space for the RVs $u,v,W$} \\
\mathcal{F}_{p} &:= \sigma(u_i,v_j,W_{ij})_{i,j=1}^p, \text{ the sigma algebra generated by these random variables} \\
\text{for }& p \in \{1,\ldots,r\}
\end{align*}
So we have that $S_r = u^\top W v$, hence it suffices to check conditions 1,2,3 in the Martingale CLT.\\
We first show 1, that $X$ is a martingale difference sequence. It is clear that $X_p$ is $\mathcal{F}_p$ measurable by definition of $\mathcal{F}_p$. To show that $X$ is integrable, we have:
\begin{align*}
    \mathbb{E}(|X_p|) 
    & \leq \mathbb{E}(|u_p W_{pp}v_p|) + \sum_{i=1}^{p-1} \mathbb{E}(|u_p W_{pi} v_i|) + \mathbb{E}(|u_i W_{ip} v_p|) \\
    & \leq \sqrt{\mathbb{E}(u_p^2 W_{pp}^2 v_p^2)} + \sum_{i=1}^{p-1} \sqrt{\mathbb{E}(u_p^2 W_{pi}^2 v_i^2)} + \sqrt{\mathbb{E}(u_i^2 W_{ip}^2 v_p^2)} \\
    & = \frac{1}{r} + (p-1)\bigg(\frac{1}{r}+\frac{1}{r}\bigg) < \infty
\end{align*}
by the inequality $\mathbb{E}(|X|)^2 \leq \mathbb{E}(X^2)$ (shown using convexity of $g:x \rightarrow x^2$ and Jensen's inequality) and independence of $u,v,W$.
Also we have:
\begin{align*}
    \mathbb{E}(X_p|\mathcal{F}_{p-1})
    &=\mathbb{E}(u_p W_{pp}v_p) + \sum_{i=1}^{p-1} \mathbb{E}(u_p W_{pi}) v_i + u_i \mathbb{E}(W_{ip} v_p)\\
    &=0
\end{align*}
since $u_p,v_p,W_{pi},W_{ip}$ are independent of $\mathcal{F}_{p-1}$ and have zero mean. Hence $X$ forms a martingale difference sequence.

To verify the next two conditions, we first prove a lemma that will help us do so. This is the generalisation of Chebyshev's inequality to higher moments:
\begin{lemma} \label{lem:chebyshev}
Suppose $X$ is a random variable with bounded $n^{th}$ moment for some $n \in \mathbb{N}$. Then \\
$\mathbb{P}(|X-\mathbb{E}(X)|>\epsilon) \leq \frac{\mathbb{E}(|X-\mathbb{E}(X)|^n)}{\epsilon^n}$ $\forall \epsilon > 0$.
\end{lemma}
\begin{proof}
Without loss of generality, assume $\mathbb{E}(X)=0$. Then
\begin{equation*}
    \mathbb{P}(|X|>\epsilon) = \mathbb{E}[\mathbb{I}(|X|> \epsilon)] =\frac{1}{\epsilon^n}\mathbb{E}[\epsilon^n\mathbb{I}(|X|> \epsilon)] \leq \frac{1}{\epsilon^n}\mathbb{E}[|X|^n\mathbb{I}(|X|> \epsilon)] \leq \frac{\mathbb{E}[|X|^n]}{\epsilon^n}
\end{equation*}
\end{proof}

Note Lemma \ref{lem:chebyshev} shows that convergence in $L^n$ implies convergence in probability. So to show conditions 2 and 3 of the martingale CLT, it suffices to show that the expectations of the quantities on the left hand sides converge to the right hand side as scalars:
\begin{enumerate}
\item [2'.] $\sum_{p=1}^r \mathbb{E}(X_p^2) \rightarrow 1$ as $r \rightarrow \infty$.
\item [3'.] $\sum_{p=1}^r \mathbb{E}(X_p^2\mathbb{I}(|X_p|>\epsilon)) \rightarrow 0$ as $r \rightarrow \infty$ $\forall \epsilon>0$ .
\end{enumerate}
Let us show 2'. In $\mathbb{E}(X_p^2)$, note that all cross terms in $\mathbb{E}(X_p^2)$ cancel since all terms have mean 0. So we have:
\begin{align*}
\mathbb{E}(X_p^2) 
&= \mathbb{E}(u_p^2 W_{pp}^2 v_p^2) + \sum_{i=1}^{p-1} \mathbb{E}(u_p^2 W_{pi}^2 v_i^2) + \mathbb{E}(u_i^2 W_{ip}^2 v_p^2) \\
&= \frac{1}{r^2}+(p-1)\bigg(\frac{1}{r^2}+\frac{1}{r^2}\bigg) = \frac{2p-1}{r^2} \\
&\Rightarrow \sum_{p=1}^r \mathbb{E}(X_p^2) = \frac{2}{r^2}\sum_{p=1}^r p - r\cdot \frac{1}{r^2}= \frac{2}{r^2}\frac{(r+1)r}{2} -\frac{1}{r} = 1
\end{align*}

To show 3', we first note that for a random variable $X$,
\begin{equation*}
\int_{\delta}^\infty \mathbb{I}(X > t) dt = (X-\delta)\mathbb{I}(X > \delta) \text{ for } \delta \in \mathbb{R}
\end{equation*}
Setting $X=X_p^2, \delta = \epsilon^2$ and rearranging we have:
\begin{align*}
X_p^2 \mathbb{I}(|X_p| > \epsilon) &= X_p^2 \mathbb{I}(X_p^2 > \epsilon^2) = \epsilon^2\mathbb{I}(X_p^2 > \epsilon^2) + \int_{\epsilon^2}^\infty \mathbb{I}(X_p^2> t) dt \\
&=\epsilon^2\mathbb{I}(|X_p| > \epsilon) + \int_{\epsilon}^\infty 2s \mathbb{I}(|X_p| > s) ds \text{ by change of variables } t=s^2 \\
\Rightarrow  \mathbb{E}[X_p^2 \mathbb{I}(|X_p| > \epsilon)]  &= \epsilon^2 \mathbb{P}(|X_p| > \epsilon) + \int_{\epsilon}^\infty 2s \mathbb{P}(|X_p| > s) ds
\end{align*}

Now we would like to use Lemma \ref{lem:chebyshev} to upper bound the right hand side. Note we want to use even $n$ such that $\mathbb{E}[|X|^n]=\mathbb{E}(X^n)$, since we know how to compute $\mathbb{E}(X_p^n)$ but not $\mathbb{E}[|X_p|^n]$. Also note that $\mathbb{P}(|X_p| > s)$ can be bounded by $\frac{\mathbb{E}(X_p^n)}{s^n}$. So we want $n>2$ for the bound on the integral to become finite. Hence we use $n=4$, and show that $\mathbb{E}(X_p^4)$ is sufficiently small so that even when we sum over $p=1,\ldots,r$, we have that the upper bound tends to 0 as $r \rightarrow \infty$.
First we compute $\mathbb{E}(X_p^4)$. Note from the multinomial theorem: 
\begin{equation*}
(x_1 + x_2  + \cdots + x_m)^n 
 = \sum_{k_1+k_2+\cdots+k_m=n} {n \choose k_1, k_2, \ldots, k_m}
  \prod_{1\le t\le m}x_{t}^{k_{t}}
\end{equation*}
where
\begin{equation*}
{n \choose k_1, k_2, \ldots, k_m}
 = \frac{n!}{k_1!\, k_2! \cdots k_m!}
\end{equation*}
Applying this to $X_p^4$ and taking the expectation, we see that the only cross terms that survive are products of even powers of the terms, namely where two of the $k_i$ are 2 and the rest are 0. \\
Noting ${n \choose 2,2}=6$, and that $\mathbb{E}(X^4)=3\sigma^4$ for $X \sim N(0,\sigma^2)$ we have:
\begin{align*}
\mathbb{E}[X_p^4]
&=\mathbb{E}[u_p^4 W_{pp}^4v_p^4 + \sum_{i=1}^{p-1} u_p^4 W_{pi}^4 v_i^4 + u_i^4 W_{ip}^4 v_p^4] \\
&+6\mathbb{E}\bigg[(u_p^2 W_{pp}^2 v_p^2)\bigg(\sum_{i=1}^{p-1} u_p^2W_{pi}^2v_i^2 +u_i^2 W_{ip}^2 v_p^2 \bigg)\bigg] \\
&+6\mathbb{E}\bigg[\sum_{i \neq j}^{p-1} u_p^2 W_{pi}^2 v_i^2 u_p^2W_{pj}^2 v_j^2 + u_i^2 W_{ip}^2 v_p^2 u_j^2 W_{jp}^2 v_p^2\bigg] \\
&+6\mathbb{E}\bigg[\sum_{i,j=1}^{p-1} u_p^2 W_{pi}^2 v_i^2 u_j^2 W_{jp}^2 v_p^2\bigg] \\
&=(2p-1) \cdot \frac{3}{r^2}\cdot 3 \cdot \frac{3}{r^2} + 6\cdot 2(p-1) \cdot \frac{3}{r^2}\cdot \frac{1}{r} \cdot \frac{1}{r}  \\
&+ 6\cdot 2 {p-1 \choose 2}\frac{3}{r^2}\cdot\frac{1}{r^2} + 6(p-1)^2\frac{1}{r^2}\cdot \frac{1}{r^2} \\
&=\frac{3}{r^4}(8p^2+8p-7)
\end{align*}
So $\mathbb{P}(|X_p|>\epsilon) \leq \frac{3}{\epsilon^4 r^4}(8p^2+8p-7)$. Hence
\begin{align*}
\mathbb{E}[X_p^2 \mathbb{I}(|X_p| > \epsilon)] 
& \leq \frac{3}{\epsilon^2 r^4}(8p^2+8p-7) + \int_{\epsilon}^\infty 2s\frac{3}{s^4 r^4}(8p^2+8p-7) ds \\
& = \frac{3}{r^4}(8p^2+8p-7)\bigg(\frac{1}{\epsilon^2} + \int_{\epsilon}^\infty \frac{2}{s^3} ds \bigg) \\
& = \frac{3C}{r^4}(8p^2+8p-7)
\end{align*}
where $\int_{\epsilon}^\infty \frac{2}{s^3} ds = C$. So 
\begin{equation*}
\sum_{p=1}^r \mathbb{E}[X_p^2 \mathbb{I}(|X_p| > \epsilon)] \leq \frac{C}{r^4}\sum_{p=1}^r 8p^2+8p-7 = O\bigg(\frac{1}{r}\bigg) \rightarrow 0  \text{ as } r \rightarrow \infty
\end{equation*}
since $\sum_{p=1}^r 8p^2+8p-7 = O(r^3)$. \\
So we have shown conditions 1,2',3', hence by martingale CLT we have that
\begin{equation*}
S_r = u^\top W v \overset{d}{\rightarrow} \mathcal{N}(0,1) \text{ as } r \rightarrow \infty
\end{equation*}
We can prove the claim for $D>2$ in a similar fashion.
\end{proof}

\section{Feature Hashing}
\label{apd:hash}
Suppose we have features $\phi(x) \in \mathbb{R}^n$. When $n$ is too large, we may use feature hashing \citep{weinberger2009feature} to reduce the dimensionality of $\phi$:
\begin{lemma}
Let $h: \{1,\ldots,n\} \rightarrow \{1,\ldots,m\}$ be a hash function for $m \ll n$. i.e. $\mathbb{P}(h(i)=j)=\frac{1}{m}$ $\forall j \in \{1,\ldots,m\}$. Also let $\xi:\{1,\ldots,n\} \rightarrow \{\pm 1\}$ be a hash function. \\ 
Define $\bar{\phi}(x) \in \mathbb{R}^m$ as follows: $\bar{\phi}_j(x)=\sum_{i:h(i)=j} \xi(i)\phi_i(x)$ \\
Then $\mathbb{E}[\bar{\phi}(x)^\top \bar{\phi}(x')]=\phi(x)^\top \phi(x')$, $Var[\bar{\phi}(x)^\top \bar{\phi}(x')]=O(\frac{1}{m})$.
\end{lemma}

\section{Random Fourier Features}
\label{apd:rff}
\begin{theorem}[Bochner's Theorem\citep{rudin1964fourier} ] 
\label{thm:bochner}
A stationary kernel k(d) is positive definite if and only if k(d) is the Fourier transform of a non-negative measure.
\end{theorem}
For RFF the kernel can be approximated by the inner product of random features given by samples from its spectral density, in a Monte Carlo approximation, as follows:
\begin{align*}
k(x-y) = \int_{\mathbb{R}^D} e^{iv^T(x-y)} d\mathbb{P}(v) \propto \int_{\mathbb{R}^D} p(v)e^{iv^T(x-y)} dv 
&= \mathbb{E}_{p(v)}[e^{iv^Tx}(e^{iv^Ty})^*] \\
&= \mathbb{E}_{p(v)}[Re(e^{iv^Tx}(e^{iv^Ty})^*)] \\
& \approx \frac{1}{n} \sum_{k=1}^n Re(e^{i{v_k}^Tx}(e^{i{v_k}^Ty})^*) \\
& = \mathbb{E}_b [\phi(x)^T \phi(y)]
\end{align*}
where $\phi(x) = \sqrt{\frac{2}{n}}(cos({v_1}^Tx+b_1),\ldots,cos({v_m}^Tx+b_n))$ with spectral frequencies $v_k$ iid samples from $p(v)$ and $b_k$ iid samples from $U[0,2\pi]$. \\
For a one dimensional squared exponential kernel $k(x,y)=\sigma_f^2\exp\Big(-\frac{(x-y)^2}{2l^2}\Big)$, the spectral density is $\mathcal{N}(0,l^{-2})$. So we use features $\phi(x) = \sigma_f \sqrt{\frac{2}{n}}(cos({v_1}^Tx+b_1),\ldots,cos({v_m}^Tx+b_n))$ where $v_k$ iid samples from $\mathcal{N}(0,l^{-2})$ and $b_k$ iid samples from $U[0,2\pi]$. \\

\section{Choice of Feature Map}
\label{apd:feature}

\textbf{Cholesky features} Consider data with inputs lying on a $D$-dimensional grid: $x_i \in \mathcal{X}=\times_{d=1}^D X^{(d)}$, $|X^{(d)}|=n_d$ finite, where $k_d(x_i,x_j)$ only depends on the values that $x_i,x_j$ take in $X^{(d)}$. The $X^{(d)}$ can be, for example, a finite set of points in Euclidean space, or the set of values a categorical variable can take. Then the Gram matrix $K$, containing the values of the kernel evaluated at each pair of points on the full grid, can be written as $K = \otimes_{d=1}^D K^{(d)}$, a Kronecker product of the Gram matrices $K^{(d)} \in \mathbb{R}^{n_d \times n_d}$ on each dimension \citep{saatcci2012scalable}. The same holds for the Cholesky factor $L$ where $K=LL^\top$: we have $L=\otimes_{d=1}^D L^{(d)}$ where $K^{(d)}=L^{(d)}L^{(d)\top} \in \mathbb{R}^{n_d \times n_d}$. Then we define $\phi_d(x_i)$ to be the $i^{th}$ row of $L^{(d)}$, so that $k_d(x_i,x_j)=K_{ij}^{(d)}=\phi_d(x_i)^\top \phi_d(x_j)$. In general a Cholesky decomposition for an $m$ by $m$ matrix takes $O(m^3)$ to compute. Thus $\phi_d(x_i)$ for $i=1,\ldots,N$ require $O(n_d^3)$ to compute in total. Hence the computation of features become feasible even for large $N$ as long as the $n_d$ are reasonably small.

\textbf{Random feature maps} In most cases the data does not lie on a grid, nor can $k_d$ be expressed as the inner product of finite feature vectors. In this case we can use random feature maps $\phi_d:\mathcal{X} \rightarrow \mathbb{R}^n$ where $\mathbb{E}[\phi_d(x)^\top\phi_d(x')]=k_d(x,x')$. An example is random Fourier features (RFF) \citep{rahimi2007random} for stationary kernels, where $\mathbb{V}[\phi_d(x)^\top\phi_d(x')]=O(\frac{1}{n})$. So we are introducing a further approximation $k_d(x,x') \approx \phi_d(x)^\top\phi_d(x')$, with more accurate approximations for larger $n$. This is feasible even for large $N$ as $\phi_d(x)$ only takes $O(n)$ computation. See Appendix \ref{apd:rff} for details. For non-stationary kernels, we can obtain features by Nystr{\"o}m methods \citep{williams2001using,drineas2005nystrom}, which use a set of $n$ inducing points to approximate $K$. The kernel is evaluated for each pair of inducing points and also between the inducing points and the data, giving matrices $K_{nn}$ and $K_{Nn}$. Then $\hat{K} \approx K_{Nn}K_{nn}^{-1}K_{Nn}^\top=\Phi^\top \Phi$ where $\Phi=L_{nn}^{-1}K_{Nn}^\top$. Hence the columns of $\Phi$ can be defined to be the Nystr{\"o}m features.

\section{Collaborative Filtering} \label{apd:cf}
\subsection{Using Binary Vectors for Side Information}
Note if the side information $\omega_1(u_i),\omega_2(v_j)$ are binary vectors with non-zeros at indices $\mathcal{I}_i,\mathcal{J}_j$ respectively, we have:
\begin{align*}
f(u_i,v_j) = (a_1 U_i + b_1\sum_{k \in \mathcal{I}_i}U_{n_1+k})^T W(a_2 V_j + b_2\sum_{k \in \mathcal{J}_j}V_{n_2+k})
\end{align*} 
which can be reparametrised to:
\begin{align} \label{eq:cfside}
f(u_i,v_j) = a(U_i + b\sum_{k \in \mathcal{I}_i}U_{n_1+k})^T W(V_j + c\sum_{k \in \mathcal{J}_j}V_{n_2+k})
\end{align} 

\subsection{Hyperparameter tuning for MovieLens 100K}
Hyperparameters were tuned on the following values. For PMF and fixed W TGP: $\sigma_u=[0.3,0.1,0.03],\sigma^2=[1.0,0.1,0.01,0.001],\epsilon_u = [10^{-5},10^{-6},10^{-7}]$ where $\epsilon_u,\epsilon_w$ are the step sizes for SGD on $U/V$ and $W$ respectively. We noticed that for a fixed $W$ the model overfits quickly in less than 30 epochs, whereas when learning $W$ the test RMSE decreases steadily. So we used a different grid of parameters for tuning the models where $W$ is learned: $\sigma_u=[0.3,0.1],\sigma^2=[1.0,0.75],\epsilon_u,\epsilon_w=[10^{-5},10^{-6}]$. For models with side information, we tuned on $a=[0.25,0.5,0.75], b,c=[0.15,0.3,0.45]$.

\section{California House Prices Data}
\label{apd:cali}

\begin{figure}[H]

\centering
  {\includegraphics[width=1.0\linewidth]{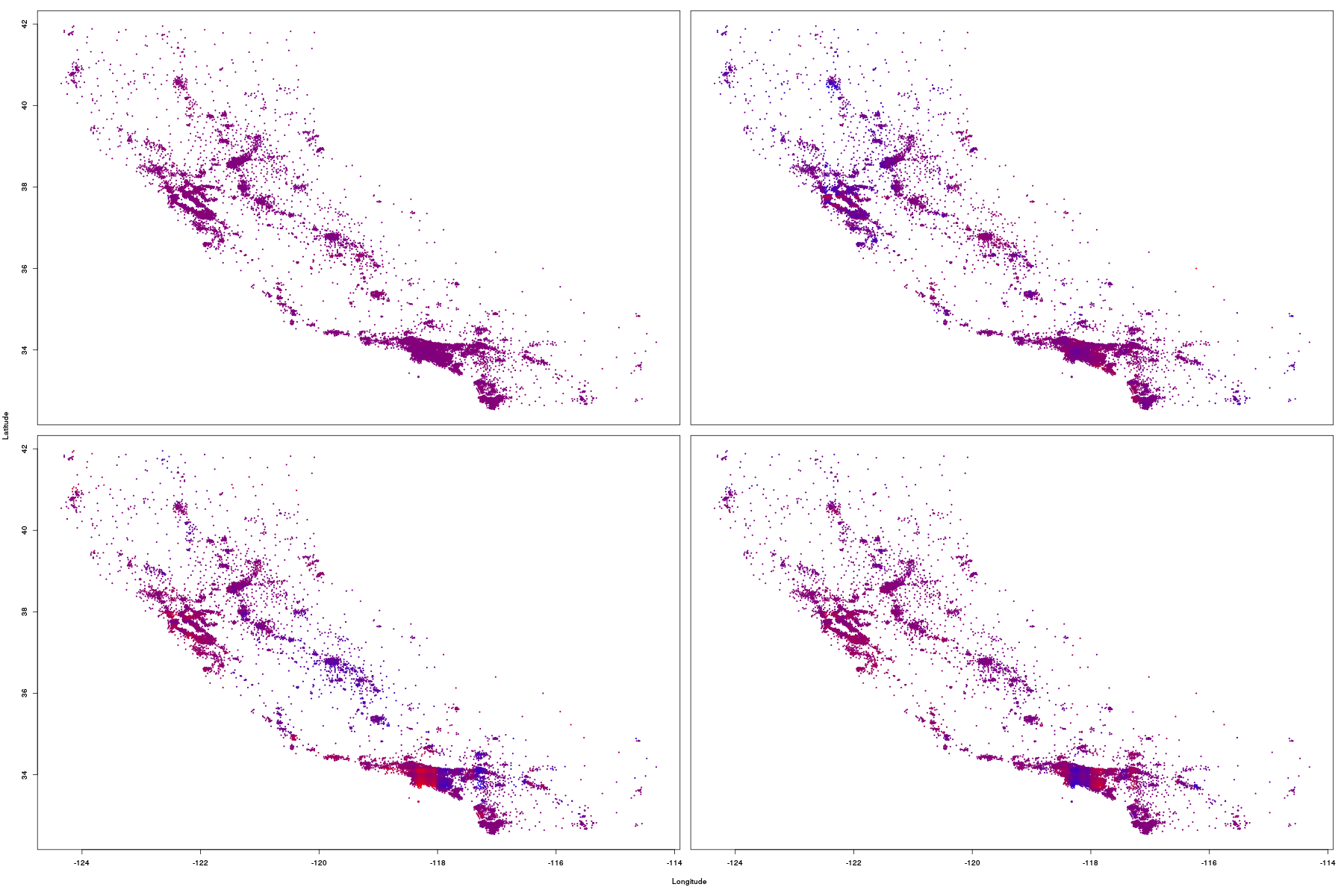}}
  \caption{Heatmap showing the four additive components of predictions of the last sample of TGP for $r=2, n=200$, using uniform colouring scheme.} \label{fig:calipredunif}
\end{figure}

\begin{figure}[H]
\centering
  {\includegraphics[width=0.8\linewidth]{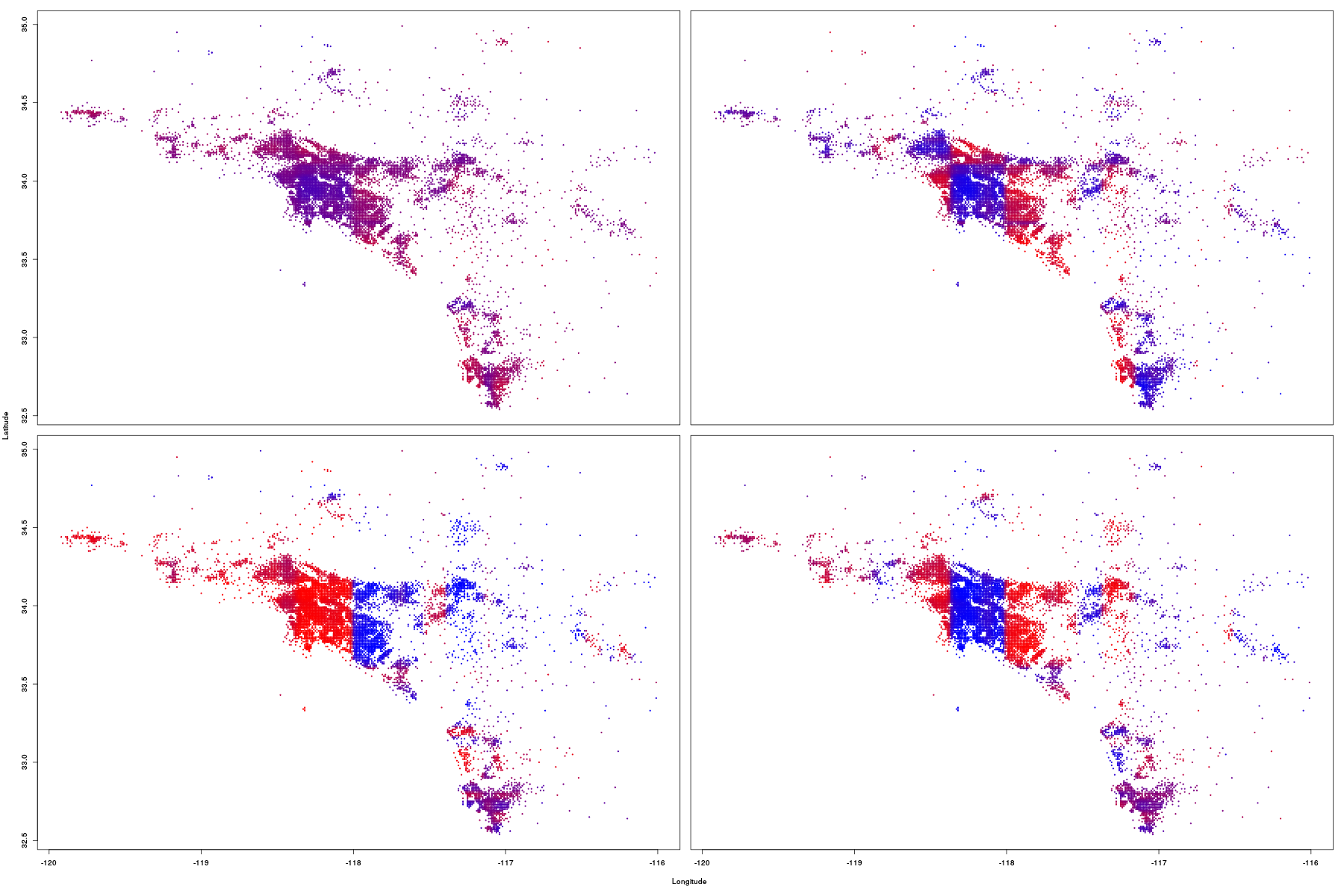}}
  \caption{Zoom in on LA area of Figure \ref{fig:calipred}.} \label{fig:calipredla}
\end{figure}

\begin{figure}[H]
\centering
  {\includegraphics[width=0.6\linewidth]{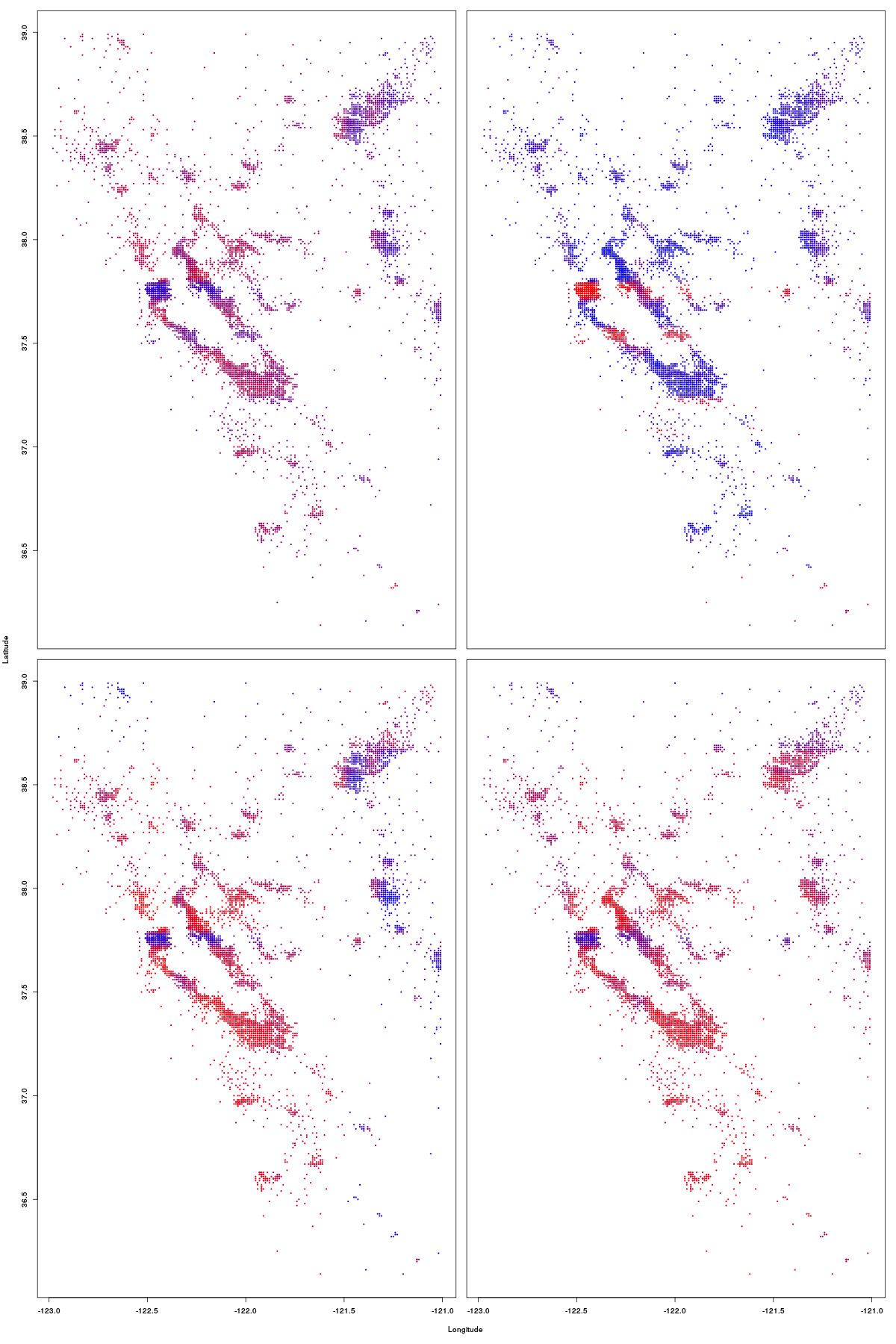}}
  \caption{Zoom in on Bay area of Figure \ref{fig:calipred}.} \label{fig:calipredbay}
\end{figure}

\begin{table}[!htpb]
\caption{Mean and standard deviation of Gelman Rubin statistic for HMC on TGP.}
\label{tab:calirhat}
\begin{center}
\begin{tabular}{|l|l|l|l|l|}
\hline
{\bf Model}         &{$n=25$}           &{$n=50$}           &{$n=100$}          &{$n=200$}\\ \hline
TGP, $r=2$       &$2.67 \pm 1.34$    &$2.55 \pm 1.37$    &$2.10 \pm 0.70$    &$1.92 \pm 0.67$\\ 
TGP, $r=5$       &$1.06 \pm 0.27$    &$1.06 \pm 0.19$    &$1.15 \pm 0.11$    &$1.11 \pm 0.11$\\ 
TGP, $r=10$      &$1.00 \pm 0.03$    &$1.02 \pm 0.04$    &$1.01 \pm 0.02$    &$1.06 \pm 0.03$\\ \hline
\end{tabular}
\end{center}
\end{table}

\begin{table}[!htpb]
\caption{Mean and standard deviation of Effective Sample Size (out of 1200) for HMC on TGP.}
\label{tab:calineff}
\begin{center}
\begin{tabular}{|l|l|l|l|l|}
\hline
{\bf Model}         &{$n=25$}           &{$n=50$}           &{$n=100$}          &{$n=200$}\\ \hline
TGP, $r=2$       &$230 \pm 459$    &$5 \pm 17$    &$11\pm 81$    &$12 \pm 74$\\ 
TGP, $r=5$       &$244 \pm 118$    &$121 \pm 70$    &$34 \pm 64$    &$42 \pm 79$\\ 
TGP, $r=10$      &$692 \pm 152$    &$196 \pm 93$    &$310 \pm 111$    &$96 \pm 165$\\ \hline
\end{tabular}
\end{center}
\end{table}

\section{Irish Wind Data}

\textbf{Regression on spatio-temporal data with grid structure} We use the Irish wind data \footnote{Obtained from \url{http://www.inside-r.org/packages/cran/gstat/docs/wind}} giving daily average wind speeds for 12 locations in Ireland between 1961 and 1978  (78,888 observations). We only use the covariates longitude, latitude and time. Note a 2D grid structure arises for the data when we treat the spatial covariates as one dimension and time as another. Again we whiten each covariate and observations, and use 20,000 randomly chosen data points for training and the rest for test. Using an isotropic SE kernel for space, and the sum of a periodic kernel and a SE kernel for time (to model annual periodicity and global trend), we first fit a GP efficiently exploiting the grid structure \citep{saatcci2012scalable}. The optimised hyperparameters are then used to construct Cholesky features. Again we use NUTS for inference on both the \textit{full-rank} model and TGP, using 4 chains with 100 warmup draws and 100 samples. 

\begin{table}[!htbp]
\centering
\begin{minipage}{\linewidth}
\centering
\caption{Train/Test RMSE on Irish wind data.}

\begin{tabular}{|l|l|l|}
\hline
{\bf Model}         &{\bf Train RMSE}   &{\bf Test RMSE} \\ \hline
GP                  &4.8822             &4.9915 \\ 
Full-rank           &4.8816             &4.9898 \\  
TGP, $r=2$          &4.9120             &4.9753 \\ 
TGP, $r=5$          &4.8996             &\textbf{4.9735} \\ 
TGP, $r=10$         &4.8913             &4.9754 \\ \hline
\end{tabular}
\label{tab:wind} 
\end{minipage}
\end{table}

All models show good convergence after 100 warmup draws, indicated by the aforementioned convergence diagnostics. Looking at Table \ref{tab:wind}, we see similar patterns in the results for the wind data as for the house prices data: the GP, which is equivalent to the \textit{full-rank} model with Cholesky features (confirmed by similar train/test RMSE), shows lower training error than TGP, whereas TGP shows superior predictive performance. These results again suggest that TGP is an effective regulariser towards simpler regression functions compared to GPs. 

\label{apd:wind}
\begin{figure}[H]
\centering
  {\includegraphics[width=1.0\linewidth]{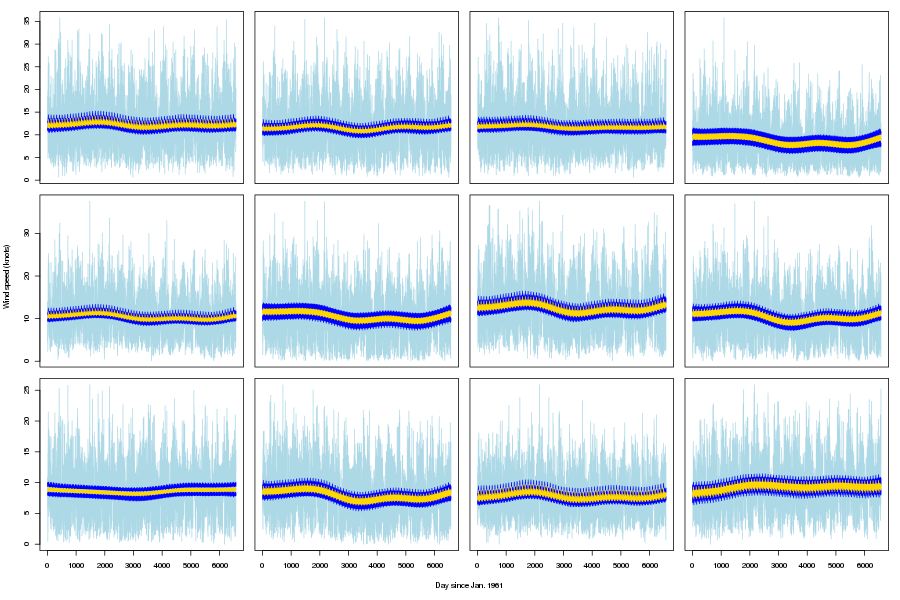}}
  {\caption{The predictions for TGP with $r=5$ on the 12 locations. The light blue lines are the true observations, the yellow are the mean predictions, and the blue show 2.5\% and 97.5\% percentiles of predictions for samples.} \label{fig:wind}}
\end{figure}

\section{Future Work}
Note that TGP can easily be extended to non-Gaussian likelihoods, since all we need for SGD and HMC is the likelihood and priors to be analytic and differentiable in the parameters. For very high dimensions where even the $r^D$ entries in $W$ are undesirable, we can use a sparse representation of $W$ with say $Q$ non-zeros. All derivations carry forward, and we obtain time complexity $O(m(nrD+QD))$ for gradient computations in SGD. It would be interesting to compare TGP against other algorithms suitable for high-dimensional data. Furthermore, it would be desirable to have a sampling algorithm that scales sub-linearly, to benefit from the Bayesian approach to learning when $N$ is large and HMC is infeasible. One example is Stochastic Gradient Langevin Dynamics (SGLD) \citep{welling2011bayesian} among many other Stochastic Gradient MCMC \citep{ma2015complete} algorithms. We have also tried mean-field variational inference, but results were poor compared to HMC. Moreover a more efficient method of tuning hyperparameters than by cross-validation would be ideal, especially for big $N$ settings with many kernel hyperparameters. One potential solution is the fully Bayesian approach, learning hyperparameters directly by imposing priors and sampling or MAP. We leave these extensions for future work.

\end{document}